\documentclass{article} 
\usepackage{times}  
\usepackage{helvet}  
\usepackage{courier}  
\usepackage{url}  
\usepackage{graphicx}  
\usepackage{algorithm}
\usepackage{algorithmic}

\usepackage{amsmath,amssymb,bm,bbold}
\usepackage{amsthm,framed,url}
\usepackage{color}
\usepackage{psfrag}
\usepackage{epstopdf}
\usepackage{wrapfig}
\usepackage{enumitem,array}
\usepackage{xfrac}
\usepackage{verbatim}

\newtheorem{theorem}{Theorem}
\newtheorem{proposition}{Proposition}
\newtheorem{lemma}{Lemma}

\theoremstyle{definition}
\newtheorem{definition}{Definition}

\newcommand{\x}{\boldsymbol{x}}
\newcommand{\z}{\boldsymbol{z}}
\newcommand{\bmu}{\boldsymbol{\mu}}
\newcommand{\bdelta}{\boldsymbol{\delta}}

\newcommand{\I}{\boldsymbol{I}}

\newcommand{\GN}{\mathcal{N}}
\newcommand{\cS}{\mathcal{S}}
\newcommand{\cM}{\mathcal{M}}
\newcommand{\cI}{\mathcal{I}}
\newcommand{\cJ}{\mathcal{J}}

\newcommand{\R}{\mathbb{R}}
\newcommand{\E}{\mathbb{E}}

\newcommand{\T}{{\!\top\!}}

\begin{document}
%
\title{On Convergence of Epanechnikov Mean Shift}
\author{
	Kejun Huang\thanks{University of Minnesota, Minneapolis, MN 55414. Email: \texttt{huang663@umn.edu}}
	\qquad
	Xiao Fu\thanks{Oregon State University, Corvallis, OR 97331. Email: \texttt{xiao.fu@oregonstate.edu}}
	\qquad
	Nicholas D. Sidiropoulos\thanks{University of Virginia, Charlottesville, VA 22904. Email: \texttt{nikos@virginia.edu}}
}
\date{}
\maketitle

\begin{abstract}
Epanechnikov Mean Shift is a simple yet empirically very effective algorithm for clustering. It localizes the centroids of data clusters via estimating modes of the probability distribution that generates the data points, using the `optimal' Epanechnikov kernel density estimator.
However, since the procedure involves \emph{non-smooth} kernel density functions,
the convergence behavior of Epanechnikov mean shift lacks theoretical support as of this writing---most of the existing analyses are based on smooth functions and thus cannot be applied to Epanechnikov Mean Shift. In this work, we first show that the original Epanechnikov Mean Shift may indeed terminate at a non-critical point, due to the non-smoothness nature. Based on our analysis, we propose a simple remedy to fix it. The modified Epanechnikov Mean Shift is guaranteed to terminate at a local maximum of the estimated density, which corresponds to a cluster centroid, within a \emph{finite} number of iterations. We also propose a way to avoid running the Mean Shift iterates from every data point, while maintaining good clustering accuracies under non-overlapping spherical Gaussian mixture models. This further pushes Epanechnikov Mean Shift to handle very large and high-dimensional data sets. 
Experiments show surprisingly good performance compared to the Lloyd's $K$-means algorithm and the EM algorithm.
\end{abstract}

\section{Introduction}\label{Sec:intro}
Clustering is a fundamental problem in artificial intelligence and statistics~\cite{jain1999data}. The simplest form is arguably the $K$-means clustering, in which a set of data points\\ $\{\x_m\}_{m=1}^M \subseteq \R^d$ is given, and the objective is to separate them into $K$ clusters, such that the sum of the cluster variances is minimized. It has been shown that $K$-means clustering is NP-hard in general~\cite{aloise2009np,dasgupta2009random}, even though Lloyd's algorithm usually gives reasonably good approximate solutions~\cite{Lloyd1982} when $d$ is small. In fact, it has been used as a standard sub-routine for more complicated clustering tasks like spectral clustering~\cite{Ng2001} and subspace clustering~\cite{elhamifar2013sparse}, despite the fact that it is not guaranteed to give the global optimal solution.

Several attempts have been made to quantify cases under which we can provably cluster the data under a probabilistic generative model. Based on the Gaussian mixture model (GMM), instead of applying Lloyd's algorithm or Expectation-Maximization~\cite{dempster1977maximum}, this line of work devises sophisticated and somewhat conceptual methods that guarantee correct estimation of the GMM parameters under additional conditions. The first of this genre, to the best of our knowledge, is the work in~\cite{Dasgupta1999}, which shows that if the Gaussian components are almost disjoint, then the deflation-type method in~\cite{Dasgupta1999} correctly clusters the data with high probability. A more practical approach is later proposed in~\cite{Dasgupta2000} with the same performance guarantee, which is a two-round variant of the EM algorithm. A number of follow-up works try to further improve the bound on how close the Gaussian components can be~\cite{Arora2001,vempala2004spectral,Brubaker2008a,Moitra2010,Belkin2010}, but most of them focus on theoretical guarantees but not practical implementations. 

On the other hand, there is a simple and effective method for clustering called Mean Shift~\cite{Fukunaga1975,Cheng1995,Comaniciu2002} that has been popular in the field of computer vision. Two main versions of Mean Shift are the Epanechnikov Mean Shift and Gaussian Mean Shift, and the details will be explained in the next section. Compared to $K$-means and Gaussian mixture model, fewer theoretical results have been presented regarding Mean Shift. Compared to Gaussian Mean Shift, even less analysis exists for Epanechnikov Mean Shift, partly due to its non-smoothness nature, despite its simplicity and effectiveness.

The main contribution of this paper is the establishment of convergence for Epanechnikov Mean Shift. There have been many convergence studies on different variants and approximations to the original Epanechnikov Mean Shift. However, to the best of our knowledge, there is no analysis that directly addresses the original Epanechnikov Mean Shift, partially because non-smoothness of the kernel employed in the method poses a hard analytical problem.
Nevertheless, analyzing the convergence behavior of the original Epanechnikov Mean Shift is of great interest since it is based on the `optimal' kernel in density estimation.
In this work, we provide detailed functional analysis and rigorous proof for Epanechnikov Mean Shift. We first show that the method indeed can get swamped at some undesired points, but with very simple modification it is guaranteed to reach a local optimum of the data density function that corresponds to a cluster centroids \emph{within finite number of iterations}.
This is the first analytical result that backs convergence behaviors of the Epanechnikov Mean Shift and the proposed modification is of practical significance.
Inspired by the idea of deflation by~\cite{Dasgupta1999} and follow-up works, we also propose a deflation-variant of the Mean Shift algorithm that is guaranteed to correctly cluster the data one group at a time under some additional conditions. This strategy saves a huge amount of computations compared to the original version and thus suits large-scale clustering problems.

\subsection{Illustrative Example}

Before we delve into convergence analysis of Epanechnikov Mean Shift, we give a simple illustrative example to showcase its effectiveness in clustering---which explains the reason why this particular method interests us.
Specifically, we test the performance of the proposed deflation-based Epanechnikov Mean Shift and some classic clustering methods, including Lloyd's $K$-means algorithm, Expectation-Maximization (EM) for Gaussian mixture models (GMM), the two-round variant of EM by~\cite{Dasgupta2000}, and the original Epanechnikov Mean Shift. The experiment is conducted in MATLAB, with the build-in implementation of $K$-means clustering and EM for GMM. Notice that these are well-implemented $K$-means / EM algorithms, with smart initialization suggested by $K$-means++~\cite{arthur2007k}, and/or various parallel implementation / multiple re-start enhancements that have been shown to work well in practice. The two-round variant of EM~\cite{Dasgupta2000} is mathematically proven to work well with high probability when the clusters are non-overlapping.

\begin{figure*}
\centering
\hspace*{-40pt}
\includegraphics[width=1.2\textwidth]{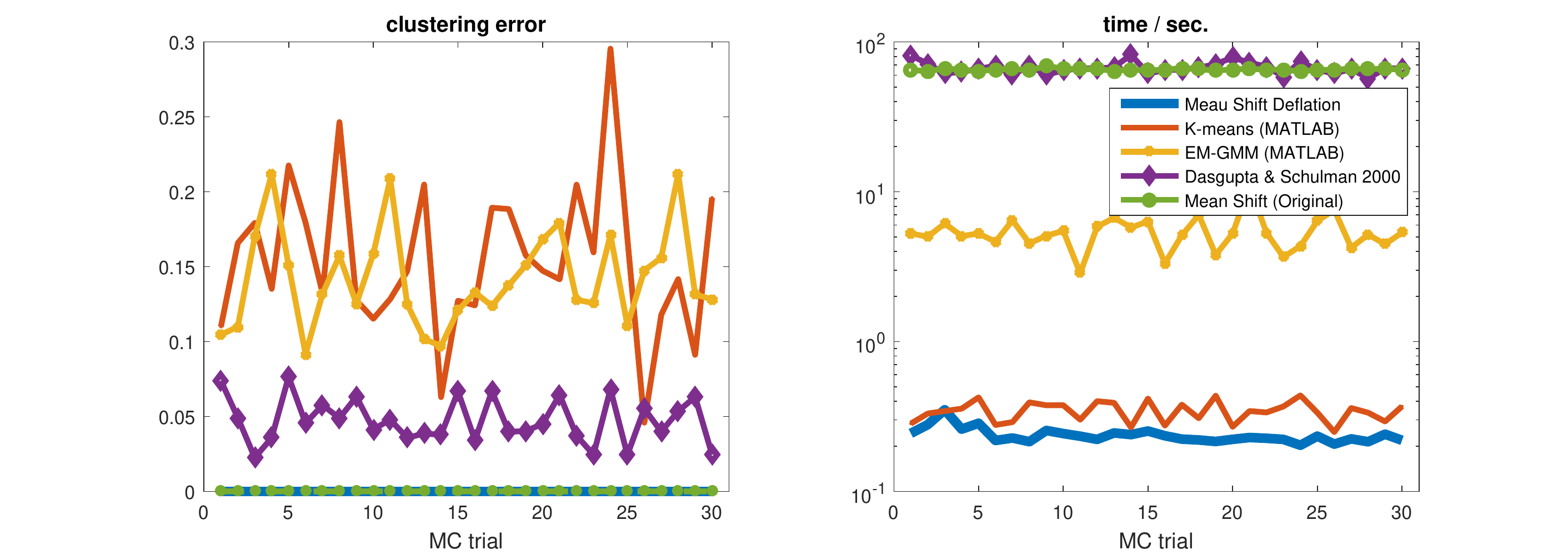}
\caption{100 Monte-Carlo simulations on synthetic data. Left: clustering error; right: run time.}
\label{fig:syn_exp}
\end{figure*}

The experiments are conducted on a synthetic dataset $\{\x_m\}_{m=1}^M\subseteq\R^d$ generated as follows. For $d=100$, we prescribe $K=30$ clusters (Gaussian components). For cluster $k$, we first randomly generate its centroid $\bmu_k\sim\GN(0,4\I)$, and then generate $M_k=50k$ i.i.d. data points from $\GN(\bmu_k,\I)$. Then we hide the cluster labels and feed the data into various algorithms. 
For the two versions of the Mean Shift algorithm, there is no need to indicate the number of clusters $K$ before hand---the only parameter (kernel bandwidth) is tuned by leave-one-out cross-validation, and then the algorithm automatically detects the number of clusters in the data set. 
For the other methods, the correct number of clusters $K$ is given, which means they are already using more information than the Mean Shift-based methods.
The procedure is repeated 30 times. In each Monte-Carlo trial, the obtained cluster labels are aligned with the true labels using the Hungarian algorithm~\cite{kuhn1955hungarian}. The clustering error is then calculated as the ratio of wrongly labeled data points over the total number of data points. 

The clustering error and runtime for each Monte-Carlo trial are shown in Figure~\ref{fig:syn_exp}.
The first observation is that plain vanilla $K$-means and EM do not cluster the data very accurately, even though there exists a good clustering structure according to how we generate the data. The clustering error is greatly reduced if we adopt the two-round variant of EM~\cite{Dasgupta2000}, with the compromise of a significantly higher amount of computation time. However, since the different Gaussian components are just marginally separated (and the sizes of each cluster are somewhat unbalanced), the performance is not as good as we might expect according to~\cite{Dasgupta2000}. Mean Shift-based methods, on the other hand, give the surprising zero clustering error in all cases; considering the fact that the correct number of clusters $K$ is not given to these methods, the results look even more impressive. In terms of computation time, the original Mean Shift takes similar time as that of~\cite{Dasgupta2000}, whereas the proposed Mean Shift deflation takes, remarkably, \emph{the least} amount of time, even compared to the simple $K$-means.

In this paper, we will study convergence properties of the original Epanechnikov Mean Shift and its deflation variant, and explain the reasons behind its effectiveness.

\section{Background: KDE and Mean Shift}\label{Sec:background}

The intuition behind Mean Shift for clustering is as follows. Suppose we have the probability density function (PDF) $p(\z)$ of the dataset $\{\x_m\}_{m=1}^M \subseteq \R^d$. If the PDF $p(\z)$ has $K$ modes, then we expect $K$ clusters in this dataset. Furthermore, if we run an optimization algorithm, e.g. gradient descent, initialized at a data point $\x_m$, and it converges to the $k$-th mode, then we declare that $\x_m$ belongs to the $k$-th cluster.

\subsection{Kernel Density Estimation}
In practice, we do not have access to the PDF $p(\z)$, but only the set of data points $\{\x_m\}_{m=1}^M$. To implement the aforementioned intuition, one needs to first estimate the PDF $p(\z)$---this is called density estimation~\cite{scott2015multivariate}. The most popular approach for density estimation is the so-called kernel density estimator (KDE). For a given kernel function $K(\z)$ that satisfies
\[
K(\z)\geq0 \text{~and~} \int K(\z)d\z = 1,
\]
the corresponding KDE is simply
\[
\widehat{p}(\z) = \frac{1}{M}\sum_{m=1}^{M}K(\z-\x_m).
\]
Two popular choices of the kernels are the Gaussian kernel
\begin{equation}\label{eq:gaussian}
K_G(\z;w) = \frac{c}{w^d}\exp\left(-\frac{\|\z\|^2}{2w^2}\right)
\end{equation}
and the Epanechnikov kernel
\begin{equation}\label{eq:epan}
K_E(\z;w) = \frac{c}{w^d}\left[1-\frac{\|\z\|^2}{w^2}\right]_+,
\end{equation}
where $c$ in~\eqref{eq:gaussian} and~\eqref{eq:epan} are normalizing constants ensuring that the kernel integrates to one.
Each of them (and all other kernels) are parameterized by a scalar $w$, called the \emph{bandwidth}, which controls the variance of the kernel. It has been shown that the Epanechnikov kernel asymptotically minimizes the mean squared error (MSE)
\begin{equation}\label{eq:mse}
\int \left(p(\z) - \widehat{p}(\z)\right)^2 d\z,
\end{equation}
among all possible kernel functions~\cite{Epanechnikov1969}. Somewhat surprisingly, this `optimal' kernel is a highly non-smooth function.

The bandwidth of the kernel $w$ plays an important role on how well the KDE approximates the true density. In practice, one can adopt the \emph{leave-one-out cross validation} approach to determine this parameter, as we did in this work. The MSE of the estimated density (and the unknown true density)~\eqref{eq:mse} can be separated into three terms:
\[
\int p^2(\z)d\z + \int \widehat{p}^2(\z)d\z -2\E\{\widehat{p}(\z)\}.
\]
The first term is unknown, but a constant; the second term can be directly calculated; and the third term is estimated via leave-one-out cross-validation. This quantity is evaluated at a set of values for $w$, and the one that gives the minimum value is selected as the bandwidth for the KDE.

\subsection{Mean Shift}
Based on the KDE $\widehat{p}(\z)$, the Mean Shift algorithm tries to find modes of $\widehat{p}(\z)$ via the following (weighted average) iterates~\cite{Fukunaga1975,Cheng1995,Comaniciu2002}, initialized at each $\x_m$:
\[
\z \leftarrow \frac{1}{\sum_{m=1}^M g(\|\z-\x_m\|^2)}
\sum_{m=1}^M g(\|\z-\x_m\|^2)\x_m,
\]
where $g(\cdot)$ is called the profile for the kernel function $K(\cdot)$. Putting details aside, the profile for the Gaussian kernel is
$\exp(\|\z-\x_m\|^2/2w^2)$,
and that for the Epanechnikov kernel is
the indicator function $\mathbb{1}(\|\z-\x_m\|^2<w^2)$. 

Existing analyses for convergence properties of Mean Shift are mostly based on smooth optimization, and argue that the update is always going at the gradient ascent direction. Borrowing the convergence results for gradient-based methods, it is then claimed that the Mean Shift iterates converges to a local maximum of $\widehat{p}(\z)$. On hindsight, we make the following comments:
\begin{enumerate}
\item The \emph{optimal} Epanechnikov kernel is non-smooth, so the existing convergence claims cannot establish convergence to a local optimum---which asymptotically approaches a mode of ${p(\bm z)}$. In fact, we will show that the plain vanilla Epanechnikov Mean Shift may indeed get stuck at a non-critical point, and we will provide a simple remedy to fix it.
\item For smooth kernels like the Gaussian kernel, it is indeed easy to show that the algorithm converges to a stationary point. However, not all stationary points are local optima---there may exist saddle points, and there is in general no simple way to check whether it is a local optimum or not.
\item An interesting observation is that Mean Shift with smooth kernels usually converges slower than Epanechnikov kernel. The convergence rate for Gaussian Mean Shift can be as slow as sub-linear~\cite{Carreira-Perpinan2007}, whereas Epanechnikov Mean Shift \emph{terminates} in finite number of steps~\cite{Comaniciu2002}, although a rigorous proof for this claim is still missing.
\end{enumerate}

The remainder of the paper tries to bridge the gap between the good empirical performance and lack of rigorous theoretical analysis for the Epanechnikov Mean Shift. We show that, with a simple modification, Epanechnikov Mean Shift terminates at a local maximum of $\widehat{p}(\z)$ within finite number of iterations. Even though the objective function is non-convex and non-smooth, the convergence result is surprisingly strong: It is guaranteed to terminate at a local optimum, \emph{never} at a saddle point, and the number of iterations is finite, with \emph{zero} precision accuracies.

For completeness, the Epanechnikov Mean Shift is clearly written in Algorithm~\ref{alg:1}. We shall call the iterative procedure between line~\ref{line:alg1-start}--\ref{line:alg1-end} ``Epanechnikov Mean Shift \emph{iterates}'', and the entire algorithm as Epanechnikov Mean Shift, which initializes the iterates at every data point $\x_m$.

\begin{algorithm}[t]
	\caption{Epanechnikov Mean Shift}
	\label{alg:1}
	\begin{algorithmic}[1]
		\REQUIRE $\{\x_m\}_{m=1}^M$, $w^2$
		\FOR{$m=1,...,M$}
		\STATE initialize $\z_m\leftarrow\x_m$
		\REPEAT[Epanechnikov Mean Shift iterates]\label{line:alg1-start}
		\STATE $\cI(\z_m) \leftarrow \{i\in[M]: \|\x_i-\z_m\|^2 < w^2 \}$
		\STATE $\displaystyle\z_m \leftarrow \frac{1}{|\cI(\z_m)|}\sum_{i\in\cI(\z_m)}\x_i$
		\UNTIL{convergence (cf. Alg.~\ref{alg:2})}\label{line:alg1-end}
		\ENDFOR
		\STATE find $K$ distinct vectors in $\{\z_m\}_{m=1}^M$, denote as $\{\bmu_k\}_{k=1}^K$
		\STATE $\x_m$ in cluster $k$ if $\z_m=\bmu_k$.
	\end{algorithmic}
\end{algorithm}

\section{Function Analysis}\label{Sec:analysis}

The Epanechnikov Mean Shift iterates tries to find modes (i.e., local maxima) of the KDE $\widehat{p}(\z)=\sum K_E(\z-\x_m;w)$ with the Epanechnikov kernel. As per conventions in the field of optimization, we define functions $\phi$ and $f$ by flipping the sign of $K_E$ and $\widehat{p}$, and omitting constants and scalings, which do not affect the task of optimization:
\begin{align}
\phi(\z) &= \min(\|\z\|^2,w^2), \label{eq:phi}\\
f(\z) &= \sum_{m=1}^{M}\phi(\x_m-\z). \label{eq:f}
\end{align}
Obviously, modes of $\widehat{p}(\z)$ correspond to local minima of $f(\z)$.
We start by analyzing the basic properties of the loss function~\eqref{eq:f}.

\begin{lemma}\label{lemma:smooth}
	The function $f(\z)$ is smooth almost everywhere.
\end{lemma}
\begin{proof}
	Notice that $f(\z)$ is a summation of component functions $\phi(\x_m-\z)$, so $f(\z)$ is smooth at $\z$ if and only if $\forall~m=1,...,M$, $\phi(\x_m-\z)$ is smooth at $\z$.	
	According to the definition of $\phi$ in~\eqref{eq:phi}, $\phi(\x_m-\z)$ is non-smooth iff $\|\x_m-\z\|^2=r$, which forms a set that has Lebesgue measure zero in $\R^d$. Because $\{\x_m\}_{m=1}^M$ is a finite set, the union set
	\[
	\{\z:\|\x_1-\z\|^2=w^2 \} \cup ... \cup \{\z:\|\x_M-\z\|^2=w^2 \},
	\]
	which forms the set of non-smooth points for $f(\z)$, also has Lebesgue measure zero. In other words, the function $f(\z)$ is smooth almost everywhere.
\end{proof}

\begin{lemma}\label{lemma:convex_smooth}
	At every smooth point $\z$ of $f(\z)$, define $\cI(\z)=\{i:\|\x_i-\z\|^2<w^2\}$, then we have
	\begin{equation}\label{eq:grad}
		\nabla f(\z) = \sum_{i\in\cI(\z)}2(\z-\x_i), 
	\end{equation}
	\begin{equation}\label{eq:hessian}
		\nabla^2 f(\z)= 2|\cI(\z)|\I. 
	\end{equation}
	Therefore, $f(\z)$ is locally convex at every smooth point, and strongly convex if $\cI$ is not empty.
\end{lemma}
\begin{proof}
	If $\z$ is a smooth point, there does not exist a $\x_j$ such that $\|\x_j-\z\|^2=w^2$. The expressions for the gradient and Hessian are elementary. For a small hyper-ball containing only smooth points, the index set $\cI(\z)$ remains the same in this convex region, therefore the Hessian remains the same in this area. Since $\nabla^2f(\z)\succeq0$, the function $f(\z)$ is locally convex. Furthermore, if $\cI(\z)\neq\emptyset$, then $\nabla^2f(\z)\succeq\I$, in which case $f(\z)$ is locally strongly convex.
\end{proof}

We now switch our focus to the non-smooth points of $f(\z)$. To study their properties, we use the concept of \emph{directional derivative}, which is defined as~\cite{bertsekas1999nonlinear}
\begin{equation}\label{eq:dir_dev}
f^\prime(\z;\bdelta) \triangleq \lim_{\alpha\downarrow0}\frac{f(\z+\alpha\bdelta)-f(\z)}{\alpha},
\end{equation}
for a particular direction $\bdelta$, if the limit exists. 
The definition~\eqref{eq:dir_dev} clearly shows that $\bdelta$ is a descending direction if $f^\prime(\z;\bdelta)<0$.
The directional derivative obeys the sum rule
\begin{equation}\label{eq:sum_rule}
f^\prime(\z;\bdelta) = \sum_{m=1}^M\phi^\prime(\x_m-\z;\bdelta).
\end{equation}
Furthermore, if $f$ is smooth at a point $\z$, then the directional derivative is simply
\(
f^\prime(\z;\bdelta) \!=\! \nabla\!f(\z)^\T\bdelta.
\)
For a non-smooth function, we can define a stationary point as follows~\cite{razaviyayn2013unified}:
\begin{definition}\label{def:stationary}
	The point $\z$ is a stationary point of $f(\cdot)$ if $f^\prime(\z;\bdelta)\geq0$ for all $\bdelta$.
\end{definition}
Notice that if $\z$ is a smooth point for $f$, Definition~\ref{def:stationary} reduces to $\nabla\!f(\z)^\T\bdelta\geq0$ for all $\bdelta$, which implies $\nabla\!f(\z)=0$, the usual definition of a stationary point for smooth functions.

\begin{lemma}\label{lemma:dd}
	The directional derivative of $f$ at $\z$ with direction $\bdelta$ is
	\begin{equation}\label{eq:dd}
		f^\prime(\z;\bdelta) = \sum_{i\in\cI(\z)}2(\z-\x_i)^\T\bdelta + 
				\hspace*{-5pt} \sum_{j\in\cJ(\z;\bdelta)} \hspace*{-5pt} 2(\z-\x_j)^\T\bdelta,
	\end{equation}
	where 
	\begin{align*}
		\cI(\z) &= \{i:\|\x_i-\z\|^2<w^2\}, \\
		\cJ(\z;\bdelta) &= \{j:\|\x_j-\z\|^2=w^2,(\z-\x_j)^\T\bdelta<0\}.
	\end{align*}
\end{lemma}
\begin{proof}
For a smooth point of $\phi(\z)$, it is easy to see that $\phi^\prime(\z;\bdelta)=2\z^\T\bdelta$ if $\|\z\|^2<w^2$, and $\phi^\prime(\z;\bdelta)=0$ if $\|\z\|^2>w^2$. For a non-smooth point when $\|\z\|^2=w^2$, we find its directional derivative by resorting to the definition~\eqref{eq:dir_dev}: $\phi^\prime(\z;\bdelta)$ equals to $2\z^\T\bdelta$ or $0$ depending on whether $\|\z+\alpha\bdelta\|^2$ is less than $w^2$ or not, when $\alpha$ goes to zero. Since
\[
\|\z+\alpha\bdelta\|^2 = \|\z\|^2 + 2\alpha\z^\T\bdelta + \alpha^2\|\bdelta\|^2
= w^2 + 2\alpha\z^\T\bdelta + o(\alpha^2),
\]
we see that $\|\z+\alpha\bdelta\|^2<w^2$ iff $\z^\T\bdelta<0$. Therefore,
\begin{equation}\label{eq:dphi}
	\phi^\prime(\z;\bdelta) = \begin{cases}
		2\z^\T\bdelta, & \!\!\|\z\|<w^2 \text{, or } \|\z\|=w^2 \text{ and } \z^\T\bdelta<0,\\
		0, &  \!\!\|\z\|>w^2 \text{, or } \|\z\|=w^2 \text{ and } \z^\T\bdelta\geq0.
	\end{cases}
\end{equation}
Now using the sum rule for the directional derivative, we conclude that $f^\prime(\z;\bdelta)$ is as given in~\eqref{eq:dd}.
\end{proof}

From the expression of $f^\prime(\z;\bdelta)$, we can show the following interesting claims:
\begin{proposition}\label{prop:not_stationary}
If $\z$ is a non-smooth point for $f(\z)$, then $\z$ cannot be a stationary point.
\end{proposition}
\begin{proof}
From the expression of $\phi^\prime(\z;\bdelta)$ in~\eqref{eq:dphi}, we see that the second term in~\eqref{eq:dd} is always $\leq0$. To prove that a non-smooth point cannot be a stationary point, we consider the following two cases:
\begin{enumerate}
\item If $\displaystyle\sum_{i\in\cI(\z)}2(\z-\x_i)\neq0$, then there exists a $\bdelta$ such that 
$\displaystyle\sum_{i\in\cI(\z)}2(\z-\x_i)^\T\bdelta<0$, e.g., 
$\displaystyle\bdelta=-\!\!\!\sum_{i\in\cI(\z)}(\z-\x_i)$, 
and thus $f^\prime(\z;\bdelta)<0$ since the second term in~\eqref{eq:dd} cannot be positive;
\item if $\displaystyle\sum_{i\in\cI(\z)}2(\z-\x_i)=0$, since $\z$ is a non-smooth point, there exists a $\bdelta$ such that $\cJ\neq\emptyset$, for example by choosing $\bdelta=-(\z-\x_j)$ for some $j$ such that $\|\z-\x_j\|^2=w^2$, then the second term in~\eqref{eq:dd} is strictly $<0$ while the first term $=0$, therefore $f^\prime(\z;\bdelta)<0$.
\end{enumerate}
To sum up, there always exists a $\bdelta$ such that $f^\prime(\z;\bdelta)<0$ when $\z$ is a non-smooth point for $f$, therefore such a point cannot be a stationary point.
\end{proof}

\begin{proposition}\label{prop:local_minimum}
A point $\z_\star$ is a local minimum for $f(\z)$ iff:
\begin{enumerate}
\item There does not exist a $\x_j$ such that $\|\x_j-\z\|^2=w^2$;
\item the set $\cI(\z_\star)=\{i:\|\x_i-\z_\star\|^2<w^2\}$ is not empty, and 
$\displaystyle\z_\star=\frac{1}{|\cI(\z_\star)|}\sum_{i\in\cI(\z_\star)}\x_i$.
\end{enumerate}
\end{proposition}
\begin{proof}
A local minimum is first of all a stationary point, which, according to Proposition~\ref{prop:not_stationary}, cannot be a non-smooth point for $f$, therefore there does not exist a $\x_j$ such that $\|\x_j-\z\|^2=w^2$.

For a smooth point, its gradient and Hessian is given in~\eqref{eq:grad} and~\eqref{eq:hessian}. Stationarity implies that $\nabla f(\z_\star)=0$, therefore $\z_\star=\frac{1}{|\cI(\z_\star)|}\sum_{i\in\cI(\z_\star)}\x_i$. If $\cI(\z_\star)$ is not empty, then $\nabla^2f(\z_\star)\succ0$, and $\z_\star$ is a local minimum; otherwise, $f(\z_\star)=Mw^2=\max f(\z)$, which is a global maximum. 
\end{proof}

\section{Convergence of Epanechnikov Mean Shift}

\begin{algorithm}[t]
	\caption{Epanechnikov Mean Shift iterates -- Redux}
	\label{alg:2}
	\begin{algorithmic}[1]
		\REQUIRE $\{\x_m\}_{m=1}^M$, $n\in[M]$
		\STATE initialize $\z^{(0)}\leftarrow\x_n$
		\LOOP
		\STATE $\cI(\z^{(t-1)}) \leftarrow \{i: \|\x_i-\z^{(t-1)}\|^2 < w^2 \}$
		\STATE $\displaystyle\z^{(t)} \leftarrow \frac{1}{|\cI(\z^{(t-1)})|}\sum_{i\in\cI(\z^{(t-1)})}\x_i$
		\IF{$\z^{(t)}=\z^{(t-1)}$}
		\STATE $\cJ(\z^{(t)}) \leftarrow \{j: \|\x_i-\z^{(t)}\|^2 = w^2 \}$
			\IF{$\cJ(\z^{(t)})=\emptyset$}
			\RETURN $\z^{(t)}$
			\ELSE
			\STATE sample $j$ from $\cJ(\z^{(t)})$
			\STATE $\displaystyle\z^{(t)} \leftarrow 
			\frac{1}{|\cI(\z^{(\!t\!-\!1\!)})|\!+\!1}\!\!
			\left( \x_j +\hspace*{-5pt}\sum_{i\in\cI(\z^{(\!t\!-\!1\!)})}\hspace*{-10pt}\x_i \right)$
			\ENDIF
		\ENDIF
		\STATE $t\leftarrow t+1$
		\ENDLOOP
	\end{algorithmic}
\end{algorithm}

Now we study the convergence of the Epanechnikov Mean Shift iterates.

\begin{lemma}\label{lemma:monotonic}
Epanechnikov Mean Shift iterates successively minimizes a local upper bound of function $f(\z)$, therefore the value of $f(\z)$ is monotonically non-increasing.
\end{lemma}
\begin{proof}
At a particular point $\widetilde{\z}$, define $\overline{f}(\z|\widetilde{\z})$ as follows:
\begin{equation*}
\overline{f}(\z|\widetilde{\z}) = \sum_{i\in\cI(\widetilde{\z})}\|\z-\x_i\|^2 + (M-|\cI(\widetilde{\z})|)w^2.
\end{equation*}
It is easy to see that 
\begin{enumerate}
\item $\overline{f}(\z|\widetilde{\z})\geq f(\z)$ for all $\z$ and $\widetilde{\z}$, 
\item $\overline{f}(\widetilde{\z}|\widetilde{\z})=f(\widetilde{\z})$, and
\item $\displaystyle
\frac{1}{|\cI(\widetilde{\z})|}\sum_{i\in\cI(\widetilde{\z})}\x_i
=\arg\min_{\z}\overline{f}(\z|\widetilde{\z})$.
\end{enumerate}
This means that, at iteration $t$, Algorithm~\ref{alg:2} updates $\z$ as 
$\displaystyle\z^{(t)}=\arg\min_{\z}\overline{f}(\z|\z^{(t-1)})$.
Therefore, 
\[
f(\z^{(t-1)})=\overline{f}(\z^{(t-1)}|\z^{(t-1)})
\geq\overline{f}(\z^{(t)}|\z^{(t-1)})
\geq f(\z^{(t)}),
\]
which means the values of $f(\z)$ obtained by Algorithm~\ref{alg:2} form a  monotonically non-increasing sequence. 
\end{proof}

Lemma~\ref{lemma:monotonic} gives the Epanechnikov Mean Shift iterates in Algorithm~\ref{alg:1} a majorization-minimization interpretation, which guarantees convergence to a stationary point in a lot of cases. Unfortunately none of the existing convergence results applies here---they either require both $f(\z)$ and $\overline{f}(\z|\widetilde{\z})$ to be smooth, or more generally 
$f^\prime(\widetilde{\z};\bdelta) = \overline{f}^\prime(\widetilde{\z}|\widetilde{\z};\bdelta)$ for all $\widetilde{\z}$ and $\bdelta$~\cite{razaviyayn2013unified}. 
Indeed, it is possible that 
$\displaystyle\z^{(t-1)}=\z^{(t)}=\arg\min_{\z}\overline{f}(\z|\z^{(t-1)})$,
in which case the algorithm converges, but there exists a $\x_j$ such that $\|\x_j-\z^{(t)}\|^2=w^2$, meaning it is a non-smooth point, thus cannot be a stationary point according to Proposition~\ref{prop:not_stationary}.
Fortunately, we find that this issue can be fixed with negligible extra computations: If such case happens, we only need to sample one $\x_j$ such that $\|\x_j-\z^{(t)}\|^2=w^2$, and then re-update $\z^{(t)}$ as the average of $\x_j$ together with all the points in $\cI(\z^{(t-1)})$. This can be instead interpreted as minimizing the following slightly different upper bound:
\begin{align*}
&\z^{(t)} =
\arg\min_{\z}\overline{f}_\sharp(\z|\z^{(t-1)}) \\
&= \|\z\!-\!\x_j\|^2 + \hspace*{-12pt}
\sum_{i\in\cI(\z^{(t-1)})}\hspace*{-15pt}\|\z\!-\!\x_i\|^2 + 
(M\!-\!|\cI(\z^{(t\!-\!1)})|\!-\!1)w^2.
\end{align*}
This elaborated procedure is fleshed out in Algorithm~\ref{alg:2}. We show that Algorithm~\ref{alg:2} provably finds a local optimum in a finite number of iterations.

\begin{lemma}\label{lemma:decrease}
The value of $f(\z)$ obtained by Algorithm~\ref{alg:2} is strictly decreasing, 
unless $\z^{(t)}=\z^{(t-1)}$. 
\end{lemma}
\begin{proof}
If $\z^{(t\!-\!1)}\neq\frac{1}{|\cI(\z^{(t\!-\!1)})|}\sum_{i\in\cI(\z^{(t\!-\!1)})}\x_i$, 
then we update
$\z^{(t)}=\frac{1}{|\cI(\z^{(t\!-\!1)})|}\sum_{i\in\cI(\z^{(t\!-\!1)})}\x_i$, and
\begin{align*}
& f(\z^{(t-1)}) - f(\z^{(t)}) \\
&\geq \overline{f}(\z^{(t-1)}|\z^{(t-1)}) - \overline{f}(\z^{(t)}|\z^{(t-1)})\\
&= \sum_{i\in\cI(\z^{(t\!-\!1)})}\|\z^{(t-1)}-\x_i\|^2 - \sum_{i\in\cI(\z^{(t-1)})}\|\z^{(t)}-\x_i\|^2 \\
&= \sum_{i\in\cI(\z^{(t\!-\!1)})} \hspace*{-5pt} \left( \|\z^{(t-1)}\|^2 - 2\x_i^\T\z^{(t-1)} - \|\z^{(t)}\|^2 + 2\x_i^\T\z^{(t)} \right)\\
&= |\cI(\z^{(t\!-\!1)})
|\left(\|\z^{(\!t\!-\!1\!)}\|^2 - 2\z^{(\!t\!)\T}\z^{(\!t\!-\!1\!)} - \|\z^{(\!t\!)}\|^2 + 2\|\z^{(\!t\!)}\|^2\right) \\
&= \left|\cI(\z^{(t-1)})\right|
\left\|\z^{(t-1)}-\z^{(t)}\right\|^2 >0.
\end{align*}
Similarly, if $\z^{(t\!-\!1)}=\frac{1}{|\cI(\z^{(t\!-\!1)})|}\sum_{i\in\cI(\z^{(t\!-\!1)})}\x_i$, 
we have
\begin{align*}
& f(\z^{(t-1)}) - f(\z^{(t)}) \\
&\geq \overline{f}_\sharp(\z^{(t-1)}|\z^{(t-1)}) - \overline{f}_\sharp(\z^{(t)}|\z^{(t-1)})\\
&= \left(|\cI(\z^{(t-1)})|+1\right)
\left\|\z^{(t-1)}-\z^{(t)}\right\|^2 \\
&= \left(|\cI(\z^{(t-1)})|+1\right)
\left\|\frac{1}{|\cI(\z^{(t\!-\!1)})|}\sum_{i\in\cI(\z^{(t\!-\!1)})}\hspace*{-10pt}\x_i \right.\\
&\hspace*{.15\textwidth}\left.- \frac{1}{|\cI(\z^{(t\!-\!1)})|+1}
\left(\x_j+\hspace*{-10pt}
\sum_{i\in\cI(\z^{(t\!-\!1)})}\hspace*{-10pt}\x_i\right)\right\|^2\\
&= \frac{1}{|\cI(\z^{(t\!-\!1)})|+1}
\left\|\z^{(t-1)}-\x_j\right\|^2 \\
&= \frac{w^2}{|\cI(\z^{(t\!-\!1)})|+1}>0.
\end{align*}
\end{proof}

\begin{theorem}\label{thm:converge}
Algorithm~\ref{alg:2} \textbf{terminates} at a local optimum of~\eqref{eq:f} in a finite number of iterations.
\end{theorem}
\begin{proof}
We first prove that Algorithm~\ref{alg:2} terminates at a local optimum of~\eqref{eq:f}. The loss function $f(\z)$ is bounded from below, and using Algorithm~\ref{alg:2}, it is monotonically non-increasing (cf. Lemma~\ref{lemma:monotonic}), so it converges to a certain value. Lemma~\ref{lemma:decrease} further shows that $f(\z)$ strictly decreases unless $\z^{(t)}=\z^{(t-1)}$, in which case $\z^{(t)}$ cannot be a non-smooth point as we showed in the proof of Lemma~\ref{lemma:decrease}. Notice that throughout the iterations $\cI(\z^{(t)})$ cannot be empty, because otherwise $f(\z)$ takes the maximum value, but since we start with $\z^{(0)}=\x_m$, 
$f(\z^{(0)})<\max f(\z)$. Invoking Proposition~\ref{prop:local_minimum}, we conclude that Algorithm~\ref{alg:2} terminates at a local optimum. 

Now suppose Algorithm~\ref{alg:2} terminates in $T$ number of iterations, we show that $T$ can be upperbounded by a finite number that only depends on the data set $\{\x_m\}$ and the bandwidth $w$ that we choose. From the proof of Lemma~\ref{lemma:decrease}, we get the (very loose) inequality
\begin{equation}\label{eq:ineq}
f(\z^{(t-1)})-f(\z^{(t)}) \geq \| \z^{(t-1)} - \z^{(t)} \|^2.
\end{equation}
Summing up both sides for $t=1,...,T$, we have
\begin{equation}\label{eq:ineq_sum}
f(\z^{(0)})-f(\z^{(T)}) \geq \sum_{t=1}^{T}\| \z^{(t-1)} - \z^{(t)} \|^2.
\end{equation}
We have shown that each of the terms on the right-hand-side is positive, unless the algorithm has terminated. If we can further find a quantity $\lambda>0$ such that
\begin{equation}\label{eq:gamma_ineq}
\| \z^{(t-1)} - \z^{(t)} \|^2 \geq \lambda > 0, \forall~t=1,...,T,
\end{equation}
then we can easily conclude that
\begin{equation}\label{eq:boundT}
T \leq \frac{f(\z^{(0)})-f(\z^{(T)})}{\lambda},
\end{equation}
which is a finite number.

To find such a $\lambda$, we note that each $\z^{(t)}$ is the average of a non-empty subset of points from the data set $\{\x_m\}$, which are all the points that can be enclosed in a Euclidean ball with radius $w$ (except for $\z^{(0)}$, which is simply one data point). 
Define
\[
\cS = \{ \cI(\z) \cup \cJ(\z;\bdelta): \forall~\z,\bdelta\in\R^d \},
\]
where $\cI(\z)$ and $\cJ(\z;\bdelta)$ are as defined in Lemma~\ref{lemma:dd}.
We know that $|\cS|<2^M$, since it is a union of subsets of $\{\x_m\}$, and there can be at most $2^M-1$ non-empty subsets of $\{\x_m\}$.
Then we define
\begin{equation}\label{eq:gamma}
\lambda = \min_{\substack{\mathcal{K},\mathcal{L}\in\cS \\ \mathcal{K}\neq\mathcal{L}}}
\left\| \frac{1}{|\mathcal{K}|}\sum_{k\in\mathcal{K}}\x_k - 
\frac{1}{|\mathcal{L}|}\sum_{\ell\in\mathcal{L}}\x_\ell \right\|^2,
\end{equation}
which exists since $\cS$ is a finite set, and it satisfies~\eqref{eq:gamma_ineq}. This $\lambda$ is also strictly positive: If
\begin{equation}\label{eq:gamma_eq}
\frac{1}{|\mathcal{K}|}\sum_{k\in\mathcal{K}}\x_k = 
\frac{1}{|\mathcal{L}|}\sum_{\ell\in\mathcal{L}}\x_\ell,
\end{equation}
the distance between this point and every point in either $\mathcal{K}$ or $\mathcal{L}$ is no greater than $w$. According to the construction of $\cS$, \eqref{eq:gamma_eq} implies $\mathcal{K}=\mathcal{L}$, contradicting $\mathcal{K}\neq\mathcal{L}$ in~\eqref{eq:gamma}.

To sum up, we have shown that there exists $\lambda>0$ that satisfies~\eqref{eq:gamma_ineq}, thus \eqref{eq:boundT} holds, meaning Algorithm~\ref{alg:2} terminates in a finite number of iterations.
\end{proof}

We remark that~\eqref{eq:boundT} is a gross over-estimate of the number of iterations. We omitted a rather big scaling factor on the right-hand-side of~\eqref{eq:ineq}, for the sake of simplicity in~\eqref{eq:ineq_sum}. The point is to show that $T$ can indeed be bounded by a finite number. In practice, our observation is that Epanechnikov Mean Shift terminates in a very smaller number of iterations, usually less than 10.

As we can see, even though we are trying to optimize a non-convex and non-smooth function, we obtain a very strong convergence result that Epanechnikov Mean Shift reaches (not \emph{approaches}) a local optimum in a finite number of iterations. From the proof one can see that the non-smoothness of the Epanechnikov kernel actually helps obtaining such a nice convergence property. It has a very different flavor as the smooth optimization based analyses. For example, the work in ~\cite{Carreira-Perpinan2007} uses a smooth Gaussian kernel, and the analysis ends up with an asymptotic convergence and a sub-linear rate in the worst case.

\section{Mean Shift Deflation for Non-overlapping Spherical Gaussian Mixtures}

The main computation bottleneck for Mean Shift is obviously the fact that the Mean Shift iterates (Algorithm~\ref{alg:2}) is run at \emph{every} data point, which is potentially a large set. In this section we provide a heuristic to avoid this computation bottleneck, under the probabilistic generative model that each cluster is generated from a spherical Gaussian with variance $\sigma^2\I$, and they are non-overlapping.

\begin{figure}[t]
	\centering
	\includegraphics[width=.6\textwidth]{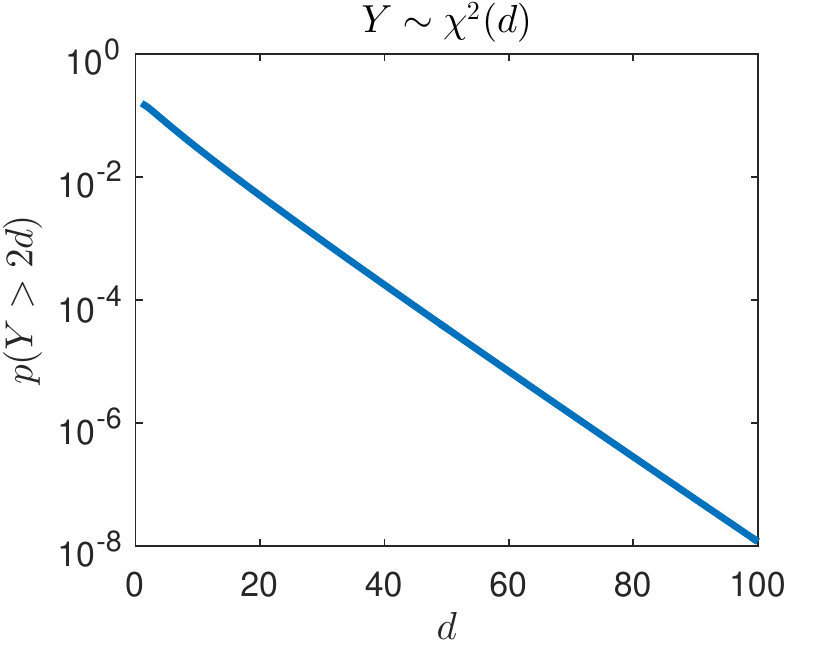}
	\caption{Probability that $Y>2d$, for $Y\sim\chi^2(d)$.}
	\label{fig:chi2}
\end{figure}

Suppose the set of data points 
$\{\x_m\}_{m=1}^M\!\subseteq\!\mathbb{R}^d$ 
comes from a mixture of Gaussian distributions 
$\GN(\bmu_k,\!\sigma^2\I)$, $\forall k\!=\!1,...,K$, 
with different means 
$\{\bmu_k\}_{k=1}^K\!\subseteq\!\R^d$ 
but the same covariance matrix $\sigma^2\I$. For a specific data point $\x_m$ coming from $\GN(\bmu_k,\!\sigma^2\I)$, the random variable $Y\!=\!\|\x_m\!-\!\bmu_k\|^2/\sigma^2$ follows the $\chi^2$-distribution with $d$ degrees of freedom, denoted as $\chi^2(d)$. 

An interesting property of the $\chi^2$-distribution is that the probability that $Y>\gamma d$ becomes almost negligible for large $d$ and $\gamma>1$. 
As an example, Figure~\ref{fig:chi2} shows the probability that $Y>2d$, as $d$ increases. 
This can be seen from the Chernoff bound on the tail probability of the $\chi^2$-distribution~\cite{dasgupta2003elementary}: for $\gamma>1$, we have that $\text{Pr}(Y>\gamma d)\leq(\gamma e^{1-\gamma})^{d/2}$. 
It can be easily shown that $\gamma e^{1-\gamma}<1$ when $\gamma>1$, implying that $\text{Pr}(Y>\gamma d)$ goes to zero \emph{at least exponentially} as $d$ increases. 

This observation inspires us to use $\sqrt{2d}\sigma$ as the bandwidth, since a Euclidean ball with radius $\sqrt{2d}\sigma$ encloses almost all points coming from $\GN(\bmu_k,\!\sigma^2\I)$ if it is centered at~$\bmu_k$. Furthermore, if $\min\|\bmu_k-\bmu_\ell\|>2\sqrt{d}\sigma$, the Gaussian components are non-overlapping, thus the ball centered at $\bmu_k$ will \emph{only} contain points coming from $\GN(\bmu_k,\!\sigma^2\I)$. Therefore, once we find one local optimum of $\widehat{p}(\z)$ that is presumably close to a $\bmu_k$, we can safely group all data points that are within radius $\sqrt{2d}\sigma$ from $\bmu_k$ and declare them as a cluster. This idea leads to the Epanechnikov Mean Shift \emph{deflation} shown in Algorithm~\ref{alg:3}.

\begin{algorithm}[t]
	\caption{Epanechnikov Mean Shift deflation}
	\label{alg:3}
	\begin{algorithmic}[1]
		\REQUIRE $\{\x_m\}_{m=1}^M$, $w^2=2d\sigma^2$
		\STATE $\cM \leftarrow \{1,2,...,M\}$, $k\leftarrow1$
		\WHILE{$\cM \neq \emptyset$}
		\STATE sample $m$ from $\cM$
		\STATE run Algorithm~\ref{alg:2} initialized at $\x_m$, outputs $\bmu_k$
		\STATE declare $\{\x_i\}_{i\in\cI(\bmu_k)}$ as cluster $k$
		\STATE $\cM\leftarrow\cM\setminus\cI(\bmu)$
		\STATE $k\leftarrow k+1$
		\ENDWHILE
	\end{algorithmic}
\end{algorithm}

As shown in Figure~\ref{fig:syn_exp}, this simple procedure obtains extremely good clustering performance, while reducing the computational complexity down to even smaller than that of Lloyd's $K$-means algorithm. Notice that the bandwidth $w$ is still estimated via leave-one-out cross-validation. Under strong generative models as in this case, however, we do observe that the estimated bandwidth $w$ is very close to~$\sqrt{2d}\sigma$.

\section{Conclusion}
We study the Epanechnikov Mean Shift algorithm, which is observed to work well but lacked theoretical analysis on its performance as of this writing. After in-depth study on estimated density $\widehat{p}(\z)$, with particular focus on its non-smoothness, we fixed an issue that could potentially affect its convergence, and showed that the Epanechnikov Mean Shift iterate terminates at a local optimum within finite number of iterations. A deflation-based variant of Epanechnikov Mean Shift is also proposed to avoid initializing the iterates at every data point, which reduces the computation considerably, and maintains good clustering performance under non-overlapping spherical Gaussian mixture assumptions.

\bibliographystyle{abbrv}
\bibliography{refs}

\end{document}